%% file: 00-main.tex
\pdfoutput=1
\documentclass[pdflatex,sn-basic]{sn-jnl}

\usepackage{amsmath, amssymb, amsthm}
\usepackage{geometry}
\usepackage{bm}
\usepackage{hyperref}

\newtheorem{definition}{Definition}
\newtheorem{theorem}{Theorem}
\newtheorem{lemma}{Lemma}
\newtheorem{proposition}{Proposition}

\title{
Counterfactual Basis Extension and Representational Geometry: An MDL-Constrained Model of Conceptual Growth
}

\author*[1]{Chainarong Amornbunchornvej}

 \affil*[1]{National Electronics and Computer Technology Center,
 112 Phahonyothin Road, Khlong Nueng, Pathum Thani, 12120, Thailand\\
 \email{chainarong.amo@nectec.or.th}}

\begin{document}

\maketitle

\begin{abstract}

Concept learning becomes possible only when existing representations fail to
account for experience. Most models of learning and inference, however,
presuppose a fixed representational basis within which belief updating occurs.
In this paper, I address a prior question: under what structural conditions can
the representational basis itself expand in a principled and selective way?

I propose a geometric framework in which conceptual growth is modeled as
admissible basis extension evaluated under a Minimum Description Length (MDL)
criterion. Experience, whether externally observed or internally simulated, is
represented as vectors relative to a current conceptual subspace. Residual
components capture systematic representational failure, and candidate
conceptual extensions are restricted to low-rank, admissible transformations.
I show that any MDL-accepted extension can be chosen so that its novel directions
lie entirely within the residual span induced by experience, while extensions
orthogonal to this span strictly increase description length and are therefore
rejected.

This yields a conservative account of imagination and conceptual innovation.
Internally generated counterfactual representations contribute to learning only
insofar as they expose or amplify structured residual error, and cannot
introduce arbitrary novelty. I further distinguish
\emph{representational counterfactuals}--counterfactuals over an agent's
conceptual basis--from causal or value-level counterfactuals, and show how MDL
provides a normative selection principle governing representational change.

Overall, the framework characterizes conceptual development as an error-driven,
geometry-constrained process of basis extension, clarifying both the role and
the limits of imagination in learning and theory change.
\end{abstract}

\input{01-intro}
\input{02-theory}
\input{03-app}
\input{04-discussion}

\section*{Acknowledgments}
The author also thanks ChatGPT (OpenAI) for assistance in improving the clarity of exposition during the writing process. All conceptual development, theoretical constructs, formal definitions, proofs, and interpretations are solely the author’s own, and all mathematical and philosophical content was independently verified by the author. NotebookLM was used to generate Figure 1 for illustrative purposes. Lastly, the author thanks Usawadee Chaiprom for her kindness and being supportive. 

\input{99-appendix}


\input{ms.bbl}

\end{document}

%% file: 01-intro.tex
\section{Introduction}

Most models of learning and inference assume that the representational basis is
fixed: updating beliefs, parameters, or hypotheses occurs \emph{within} a given
feature space or hypothesis class. In this paper, I study a prior question: how
the \emph{representational basis itself} can change in a constrained and
selective way.

I propose a geometric model in which new conceptual directions arise through
\emph{counterfactual basis extension}: internally simulated representational
changes evaluated under a Minimum Description Length (MDL) criterion and
restricted by admissibility constraints. The model treats imagination as
counterfactual exploration \emph{over representations} rather than over values
or interventions on fixed variables. Conceptual novelty is conservative: it is
admitted only when a low-dimensional extension yields a net compression gain.

My point of departure is the cognitive-geometric perspective that meaning and
inference can be analyzed in terms of positions, projections, and
representation-dependent transformations in structured spaces
\cite{amornbunchornvej2025interpretation,gardenfors2004conceptual,gardenfors2014geometry,kriegeskorte2013representational,lappin2023geometrical}.
Within this viewpoint, learning corresponds not only to moving within a space,
but also to changing which directions are representationally available.

\paragraph{Contributions.}
\begin{itemize}[leftmargin=*]
  \item I formalize an admissible class of basis extensions capturing
        continuity/invariants and bounded growth, making explicit what kinds of
        conceptual change are permitted.
  \item I give an MDL-based acceptance rule for basis extension, separating
        \emph{genuine novelty} from arbitrary reparameterization.
  \item I introduce a representation-level notion of counterfactual reasoning
        (counterfactuals over \emph{spaces}), and position it as complementary to
        causal/value-level counterfactual frameworks.
\end{itemize}

\paragraph{Scope.}
This paper focuses on representational structure and selection principles, not on
phenomenology or consciousness. I also do not claim that large language models
explicitly optimize MDL or maintain explicit conceptual subspaces; rather, I use
contemporary representation learning as a motivating comparison class for the
kinds of constraints the theory isolates
\cite{saxe2019mathematical,geva2022transformer,piantadosi2024concepts}.

The results should be read as necessary conditions on representational change
under compression-based selection, not as a descriptive model of human
cognition.

\section{Related Work}

\subsection{Geometric and Vector-Space Models of Concepts}

A long tradition in cognitive science and semantics represents meanings as points
or regions in geometric spaces rather than as abstract symbols. In
\emph{conceptual spaces} theory, concepts are modeled as regions in continuous
metric spaces spanned by cognitively meaningful quality dimensions such as
color, shape, or size \cite{gardenfors2004conceptual,gardenfors2014geometry}.
Within such frameworks, similarity corresponds to spatial proximity and
categorization to region membership, yielding accounts that align well with
empirical findings on human concept learning and typicality effects.

This geometric perspective has gained renewed support from recent work arguing
that high-dimensional vector representations are particularly well suited to
capture the richness and flexibility of human concepts. Vector-based models can
encode graded similarity, compositional structure, and systematic relations
between concepts in ways that are difficult to achieve with purely symbolic
representations \cite{piantadosi2024concepts}. Such representations provide a
natural substrate for abstraction, generalization, and conceptual combination.

Beyond cognitive theory, \emph{representational geometry} has become a central
tool in neuroscience and artificial intelligence for comparing internal
representations across brains, models, and behavior. Techniques such as
representational similarity analysis characterize representations via distance
or similarity structure, enabling principled geometric comparisons across
systems \cite{kriegeskorte2013representational}. Subsequent work has shown that
learning systematically reshapes representational geometry in both biological
and artificial systems, linking changes in geometry to behavioral performance
and task structure
\cite{greco2024predictive,wei2025representational}.

My framework aligns with these geometric approaches but addresses a distinct
question. Building on a cognitive--geometric account in which interpretation is
modeled as a linear transformation between representational spaces
\cite{amornbunchornvej2025interpretation}, I ask when and why the
\emph{representational space itself} should be extended with new dimensions.
Whereas prior work typically assumes a fixed representational basis and studies
the geometry \emph{within} it, I focus on principled conditions for
representational growth rather than representational arrangement alone.

\subsection{Representation Learning and Feature Discovery}

In machine learning, a large literature studies how internal representations and
high-level features emerge automatically from data. Deep neural networks, in
particular, have been shown to develop vector representations that mirror
semantic structure observed in human cognition
\cite{saxe2019mathematical}. Theoretical analyses demonstrate that learning
dynamics in deep networks can recapitulate patterns of semantic development,
including hierarchical differentiation and category structure.

Empirically, many studies document that as training progresses, hidden layers
come to encode increasingly interpretable features, such as object categories
or linguistic distinctions. This suggests that useful representational
dimensions can emerge in an unsupervised or weakly supervised manner under
predictive or compressive objectives. Complementary work seeks to interpret and
simplify such representations. For example, methods that reduce contextualized
embeddings to static vector spaces reveal latent semantic axes embedded within
complex models \cite{bommasani2020interpreting}. Mechanistic analyses of
transformer architectures further suggest that certain components function as
structured storage or retrieval mechanisms, with identifiable key--value
structure \cite{geva2022transformer}.

Most of this literature focuses on optimizing parameters or interpreting features
within a fixed representational architecture. By contrast, my work addresses a
higher-level question: when should the representation itself acquire new basis
dimensions? Rather than assuming a static feature space, I propose a normative
principle governing representational expansion, shifting attention from weight
adaptation to basis selection under explicit complexity cost.

\subsection{MDL and Simplicity Principles}

The Minimum Description Length (MDL) principle, originally introduced by~\cite{rissanen1978modeling} and later developed systematically
\cite{grunwald2007minimum}, formalizes Occam’s razor by selecting
representations that optimally trade off goodness of fit against
representational complexity. Classical applications of MDL include model
selection, hypothesis testing, and structure discovery, where models are
preferred if they yield shorter joint descriptions of model and data. Beyond
these settings, MDL has also been applied as an operational criterion for
identifying parsimonious structure in real-world, multi-resolution datasets
\cite{amornbunchornvej2021identifying}.

In the present work, I extend MDL from parameter selection within a fixed model
class to \emph{representational selection}. Rather than assuming a static
hypothesis space, I use MDL to evaluate hypothetical extensions of the
representational basis itself. In this role, MDL serves two functions. First, a
new representational dimension is accepted only if it yields a strict reduction
in total description length. Second, representational growth incurs an explicit
complexity cost, ruling out arbitrary or unconstrained forms of creativity.

This use of simplicity constraints is adjacent in spirit to broader concerns about how learning systems acquire compact and robust internal structure from data, particularly in light of dataset design and usage practices~\cite{paullada2021data}.
My focus, however, is on the normative conditions
under which representational change is rationally warranted given a fixed
dataset. MDL thus functions as a principled gatekeeper for conceptual growth,
ensuring that representational expansion remains conservative and data-driven.

\subsection{Counterfactual Reasoning}

Classical treatments of counterfactuals analyze hypothetical alternatives to
reality using fixed variables and possible worlds. In philosophy, counterfactual
claims are evaluated by considering the closest possible worlds in which the
antecedent holds \cite{lewis1973counterfactuals}. In causal inference,
counterfactual queries are formalized as interventions on structural causal
models, allowing predictions of alternative outcomes under fixed causal
relations \cite{pearl2009causality,woodward2003making}. Psychological research
emphasizes that human counterfactual thinking often focuses on salient deviations
from reality, such as near misses or norm violations
\cite{kahneman1986norm,roese1997counterfactual}. Recent philosophical critiques
have argued that artificial imagination cannot be genuinely unconstrained, but
must remain bounded by its technical and ontological
conditions~\cite{hui2023imagination}. I provide a formal counterpart to this
intuition, showing that under MDL constraints, representational imagination is
necessarily limited by residual structure.

My contribution is orthogonal to these traditions. I introduce
\emph{representational counterfactuals}: counterfactuals over the agent’s
conceptual basis itself rather than over variable values within a fixed
representation. A representational counterfactual asks whether introducing a new
representational dimension would improve the agent’s ability to compress and
explain experience. By evaluating such hypothetical basis extensions under an
MDL criterion, my framework provides a principled account of how imagination and
counterfactual reasoning can drive conservative yet meaningful conceptual
growth.

%% file: 02-theory.tex
\section{Theory: Counterfactual Basis Extension}

This section formalizes learning and imagination as a unified geometric process:
\emph{basis extension driven by compression pressure}. Both external experience
and internally generated data are treated as representational inputs evaluated
under a common Minimum Description Length (MDL) criterion.

\subsection{Conceptual Space and Residual Structure}

I model an agent’s cognitive state as a finite-dimensional conceptual space.

\begin{definition}[Conceptual Space]
A \emph{conceptual space} is a finite-dimensional real vector space
\[
\mathcal{C} = \mathrm{span}\{b_1,\dots,b_k\},
\]
whose basis vectors represent the agent’s available conceptual distinctions.
\end{definition}

Incoming experience is evaluated relative to this space.

\begin{definition}[Residual]
Given an experience vector $u\in\mathcal H$ and orthogonal projection
$\Pi_{\mathcal C}$ onto $\mathcal C$, the \emph{residual} is
\[
r_{\mathcal C}(u):=u-\Pi_{\mathcal C}u.
\]
\end{definition}

Residuals capture structured representational failure: components of experience
that cannot be expressed using existing conceptual distinctions.

\subsection{MDL-Governed Conceptual Change}

Learning is governed by a Minimum Description Length objective.

\begin{definition}[MDL Acceptance]
An admissible extension $\mathcal C'\supseteq\mathcal C$ is \emph{MDL-accepted} if
\[
L(\mathcal C';D)<L(\mathcal C;D),
\]
where
\[
L(\mathcal S;D)
=
\sum_{u\in D}\ell(\|u-\Pi_{\mathcal S}u\|^2)+\lambda\dim(\mathcal S),
\]
$\ell$ is nondecreasing, and $\lambda>0$ penalizes representational complexity.
\end{definition}

\subsection{Counterfactual Basis Extension}

\begin{theorem}[Counterfactual Basis Extension under MDL]
\label{thm:cbe}
Let $\mathcal C\subseteq\mathcal H$ be the agent’s current conceptual subspace and
$D$ a finite multiset of experience vectors. Define the residual span
\[
W:=\mathrm{span}\{r_{\mathcal C}(u):u\in D\}\subseteq\mathcal C^\perp.
\]

Where $C^\perp$ is the set of all vectors orthogonal to every vector in $C$. Under mild structural admissibility conditions on admissible extensions
(formalized in Appendix~\ref{app:ptc-proof}), any MDL-accepted conceptual update
must satisfy:
\begin{enumerate}
\item \emph{(Low-rank novelty)} Novel conceptual dimensions are bounded in number.
\item \emph{(Residual-supported acceptance)} Accepted novelty directions can be
chosen to lie entirely in $W$.
\item \emph{(No orthogonal gain)} Novelty orthogonal to $W$ strictly increases
description length and is never MDL-accepted.
\end{enumerate}
\end{theorem}

\paragraph{Interpretation.}
Theorem~\ref{thm:cbe} formalizes learning as an error-driven geometric process.
Conceptual growth occurs only through the resolution of structured residual
error. New concepts are not introduced arbitrarily, but inherit their structure
from systematic failures of prediction.

\begin{figure}[t]
  \centering
  \includegraphics[width=\linewidth]{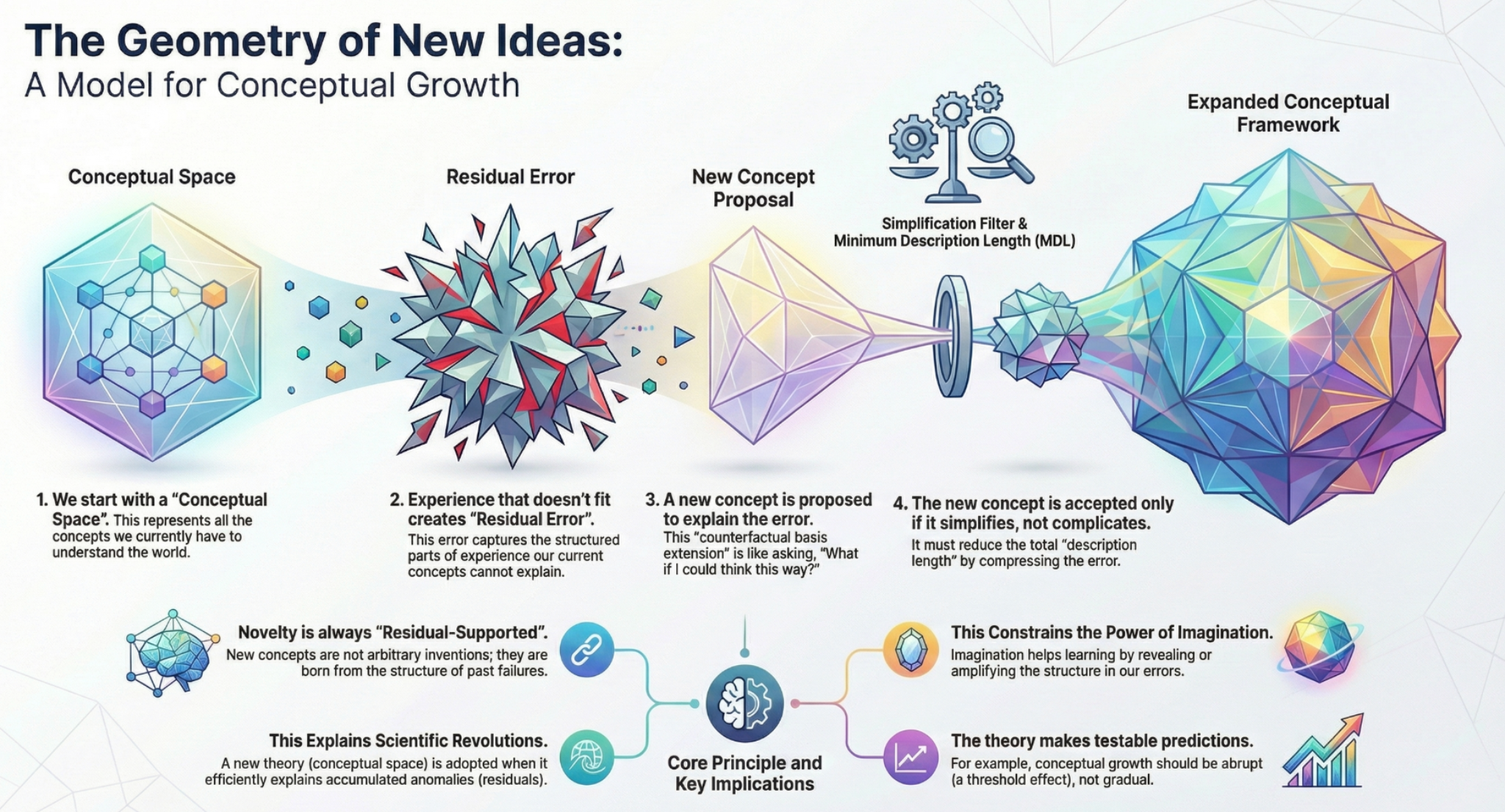}
  \caption{Schematic illustration of counterfactual basis extension under MDL.
  Residual error induces candidate novelty directions, which are filtered by
  description length reduction. Conceptual growth occurs only when a proposed
  basis extension compresses structured residuals.}
  \label{fig:geometry}
\end{figure}

\subsection{When Is Novelty Accepted?}

\begin{proposition}[MDL Acceptance Threshold]
\label{prop:mdl-threshold-general}
A one-dimensional admissible extension is MDL-accepted if and only if the
aggregate reduction in residual loss exceeds the fixed complexity penalty
$\lambda$.
\end{proposition}

\paragraph{Intuition.}
MDL learning implements a threshold rule: new conceptual dimensions are added
only when they explain enough unresolved structure to justify their descriptive
cost. The Proposition~\ref{prop:mdl-threshold-general} is also formalized in Appendix~\ref{app:ptc-proof}.

\subsection{Imagination and Counterfactual Experience}

\begin{proposition}[Two Mechanisms by which Imagination Affects Learning]
\label{prop:two-mechanisms-imagination1}
Internal simulation influences MDL-based learning in exactly two ways:
\begin{enumerate}
\item \emph{Directional enrichment}: simulation expands the residual span.
\item \emph{Threshold amplification}: simulation strengthens existing residual
directions until they cross the MDL acceptance threshold.
\end{enumerate}
\end{proposition}

\paragraph{Interpretation.}
Imagination contributes to learning only insofar as it alters the geometry or
magnitude of residual error. It neither bypasses nor overrides the constraints
imposed by compression. The Proposition~\ref{prop:two-mechanisms-imagination1} is also formalized in Appendix~\ref{app:ptc-proof}.

\paragraph{Example (Two roles of imagination).}
Let $\mathcal H=\mathbb R^3$ and let the current conceptual space be
$\mathcal C=\mathrm{span}\{(1,0,0)\}$. Suppose external experience
$D_{\mathrm{ext}}=\{(1,2,0)\}$ yields a residual
\[
r_{\mathcal C}((1,2,0))=(0,2,0),
\]
so that
\[
W_{\mathrm{ext}}=\mathrm{span}\{(0,1,0)\},
\qquad
\sum_{u\in D_{\mathrm{ext}}}\|r_{\mathcal C}(u)\|^2 = 4 < \lambda.
\]
No admissible one-dimensional extension is therefore MDL-accepted on external
experience alone.

First, suppose internal simulation generates an additional experience
$(1,0,3)$, producing a residual $(0,0,3)$. The residual span expands to
\[
W_{\mathrm{all}}=\mathrm{span}\{(0,1,0),(0,0,1)\},
\]
introducing a new residual-supported candidate novelty direction
(\emph{directional enrichment}).

Alternatively, suppose internal simulation generates two additional instances
of $(1,2,0)$. The residual span remains unchanged, but the aggregate residual
energy along $(0,1,0)$ becomes
\[
\sum_{u\in D_{\mathrm{all}}}\|r_{\mathcal C}(u)\|^2 = 12 > \lambda,
\]
so the corresponding one-dimensional extension becomes MDL-accepted
(\emph{threshold amplification}).

This example illustrates that imagination affects learning either by enlarging
the residual span or by strengthening existing residual structure until it
justifies representational growth under MDL.

\subsection{Canonical Choice of an MDL-Accepted Basis Extension}

The results above characterize which forms of conceptual novelty are admissible
under an MDL criterion. We now make this characterization more explicit by
showing that, among admissible one-dimensional extensions, the MDL objective
induces a canonical choice of novelty direction. This construction is not
algorithmic in intent, but normative: it specifies how representational
refinement would proceed if governed solely by compression-based rational
selection.

Let $\mathcal C \subseteq \mathcal H$ be the current conceptual space and
$D=\{u_1,\dots,u_n\}$ a finite multiset of experience vectors. As before, define
the residuals
\[
r_i := r_{\mathcal C}(u_i) = u_i - \Pi_{\mathcal C}u_i \in \mathcal C^\perp,
\]
and let
\[
W := \mathrm{span}\{r_i : u_i \in D\}
\]
denote the residual span induced by experience.

Assume $\ell(x)=x$ for concreteness. Consider admissible one-dimensional
extensions of the form
\[
\mathcal C' = \mathcal C \oplus \mathrm{span}\{v\},
\qquad
\|v\|=1,\; v \perp \mathcal C.
\]
For any such $v$, the change in the MDL objective satisfies
\[
L(\mathcal C';D)
=
L(\mathcal C;D)
-
\sum_{u\in D} \langle r_{\mathcal C}(u), v\rangle^2
+
\lambda.
\]
Thus, among all admissible one-dimensional extensions, the MDL objective is
minimized by choosing $v$ to maximize the aggregate squared alignment with the
residuals.

Equivalently, let
\[
\Sigma := \sum_{u\in D} r_{\mathcal C}(u)\, r_{\mathcal C}(u)^\top
\]
denote the residual covariance operator on $\mathcal C^\perp$, and let $v^\star$
be a unit eigenvector corresponding to its largest eigenvalue. Defining
\[
\mathcal C^\star := \mathcal C \oplus \mathrm{span}\{v^\star\},
\]
we obtain an admissible extension that minimizes $L(\mathcal C';D)$ among all
one-dimensional extensions.

By Proposition~1, this extension is MDL-accepted if and only if
\[
\sum_{u\in D} \langle r_{\mathcal C}(u), v^\star\rangle^2 > \lambda.
\]

\paragraph{Interpretation.}
This construction sharpens the role of residual structure in conceptual growth.
The MDL principle does not merely filter proposed extensions; it induces a
preferred direction of representational refinement aligned with the dominant
structure of residual error. Conceptual novelty is therefore neither arbitrary
nor unconstrained: when growth occurs, it proceeds along directions that inherit
their geometry from systematic failures of the existing representation.

Although the construction admits a computational reading, its significance here
is normative rather than procedural. It characterizes the unique direction along
which conceptual refinement is rationally warranted under an MDL constraint,
rather than prescribing an explicit learning mechanism.

\paragraph{Example (Canonical novelty direction in a simple setting).}
Let $\mathcal H=\mathbb R^2$ and let the current conceptual space be
$\mathcal C=\mathrm{span}\{(1,0)\}$. Consider a dataset
$D=\{u_1,u_2\}$ with
\[
u_1=(3,4),\qquad u_2=(1,2).
\]
Projection onto $\mathcal C$ removes the first coordinate, yielding residuals
\[
r_{\mathcal C}(u_1)=(0,4),\qquad r_{\mathcal C}(u_2)=(0,2).
\]
The residual span is therefore
\[
W=\mathrm{span}\{(0,1)\}.
\]

For any admissible one-dimensional extension
$\mathcal C'=\mathcal C\oplus\mathrm{span}\{v\}$ with $v\perp\mathcal C$ and
$\|v\|=1$, the MDL improvement term equals
\[
\sum_{u\in D}\langle r_{\mathcal C}(u),v\rangle^2.
\]
This quantity is maximized by $v^\star=(0,1)$, which aligns with the unique
direction in $W$. The corresponding extension
$\mathcal C^\star=\mathbb R^2$ is MDL-accepted if and only if
\[
16+4 > \lambda.
\]

This example illustrates that, when residual error exhibits a dominant
low-dimensional structure, the MDL criterion singles out a unique, residual-aligned
direction of conceptual refinement.

%% file: 03-app.tex
\section{Applications and Implications}

The framework of counterfactual basis extension yields implications across
cognitive science, philosophy of mind, philosophy of science, and the analysis
of contemporary artificial representational systems. In each case, the central
contribution is to make explicit the geometric and informational constraints
under which conceptual growth is possible.

\subsection{Concept Formation and Surprise}

In cognitive science, concept formation is often modeled as the gradual
adjustment of representations in response to prediction error or surprise~\cite{clark2015surfing}.
The present framework sharpens this picture by distinguishing mere error
from \emph{structurally informative} error.

Within my model, the residual span $W$ captures precisely those directions of
experience that cannot be compressed by the current conceptual space.
Theorem~\ref{thm:cbe} shows that only such residual-supported structure can
justify the introduction of new conceptual dimensions under an MDL criterion.
As a consequence, not all surprises are equal: isolated or unstructured
prediction failures do not drive concept formation, whereas systematic,
low-dimensional residual patterns do.

This provides a geometric account of why some anomalies trigger learning while
others are ignored or treated as noise, and offers a formal bridge between
predictive-processing accounts of surprise and representational theories of
concept acquisition.

\subsection{The Limits of Imagination}

The framework also yields a principled account of the limits of imagination.
Imagination is often treated as a faculty capable of freely generating novel
ideas or possibilities. By contrast, Theorem~\ref{thm:cbe} and
Proposition~\ref{prop:two-mechanisms-imagination1} show that internally generated
experience is subject to the same structural constraints as external data.

In particular, imagination cannot introduce genuinely new conceptual directions
unless the internally generated residuals either expand the residual span or
amplify existing residual structure sufficiently to overcome the MDL acceptance
threshold. Counterfactuals that are already well explained by the current
conceptual space, or that introduce only unstructured variation, do not support
conceptual growth.

This yields a conservative but precise notion of imaginative productivity:
imagination is epistemically effective only insofar as it reveals or sharpens
previously unresolved structure.

\subsection{Theory Change and Scientific Representation}

In the philosophy of science, theory change is frequently associated with the accumulation of anomalies within an existing research programme and the eventual introduction of new theoretical principles~\cite{lakatos1978methodology}.
The present framework provides a formalization of this process that
does not appeal to sociological factors or paradigm shifts.

On my account, a scientific theory corresponds to a conceptual space that
compresses observational data. Anomalies correspond to residual structure
outside that space. Theory change occurs precisely when residual-supported
directions are sufficiently structured and informative to justify an expansion
of the representational basis under MDL.

This perspective explains both the conservatism of mature theories and the
conditions under which radical change becomes rational: new theoretical
dimensions are introduced only when existing representations systematically
fail in ways that admit low-complexity repair.

\subsection{Interpretive Implications for Large Language Models}

This application should be read as an interpretive lens rather than an empirical claim about training dynamics. 
I do not claim that large language models explicitly optimize the MDL objective
or maintain explicit conceptual subspaces. Rather, I use them as a contemporary
illustration of how large-scale representational systems may display patterns of
generalization and failure consistent with the geometric constraints identified
by my theory.

From this perspective, phenomena such as successful analogical reasoning,
prompt sensitivity, and characteristic failure modes can be interpreted as
reflecting the geometry of residual structure induced by training data.
The framework helps explain why such models exhibit impressive recombinatory
creativity while remaining constrained in their ability to generate genuinely
novel abstractions: representational extensions are effectively limited to
directions supported by structured residuals.

Importantly, this interpretation is purely conceptual rather than mechanistic.
Its value lies not in explaining how such systems are trained, but in clarifying
the limits and structure of their apparent conceptual generalization.

\subsection{Summary}

Across these domains, the central lesson is the same: conceptual growth is not a
matter of arbitrary invention, but of structurally constrained basis extension.
Whether driven by external experience, internal simulation, or large-scale data,
new concepts emerge only where existing representations fail in coherent,
compressible ways.

%% file: 04-discussion.tex
\section{Discussion}

\subsection{Explanatory Power and Falsifiable Signatures}

Although the framework is abstract, it yields several concrete and falsifiable
signatures concerning when conceptual growth should occur and when it should
fail. These signatures distinguish counterfactual basis extension from both
unconstrained creativity and standard learning within a fixed representation.

\paragraph{Signature 1: Residual-aligned novelty.}
When conceptual expansion occurs, newly adopted basis directions should align
with dominant directions in the residual span $W$ induced by prior experience.
Equivalently, novel concepts should correspond to low-dimensional structure in
systematic prediction error rather than to arbitrary perturbations of the
existing representation. Empirically, this predicts that new representational
axes will correlate with principal components of residual error rather than
with unexplained variance orthogonal to prior failures.

\paragraph{Signature 2: Absence of growth under perfect fit.}
If the current conceptual space yields zero (or negligibly small) residuals on
the relevant experience distribution, no admissible basis extension can be
MDL-accepted. In such regimes, increasing exposure to already well-explained
data should refine parameter values but not induce representational expansion.
This distinguishes concept formation from incremental learning and predicts
periods of representational stasis in environments that are fully predictable
under existing distinctions.

\paragraph{Signature 3: Threshold-dependent concept formation.}
The MDL acceptance condition implies a sharp threshold effect: candidate basis
directions that partially explain residual structure may remain unadopted until
the aggregate explanatory gain exceeds the fixed complexity penalty. As a
result, conceptual growth should exhibit discontinuities rather than gradual
drift, with new dimensions appearing abruptly once sufficient structured
evidence accumulates. Increasing the complexity penalty $\lambda$ should delay
or suppress such transitions without changing the geometry of residual error
itself.

\paragraph{Signature 4: Two distinct roles for imagination.}
Internally generated experience can affect learning in exactly two ways:
either by introducing residual directions not present in external data
(direction enrichment), or by amplifying existing residual structure so as to
cross the MDL acceptance threshold (threshold amplification). Importantly,
simulation that neither expands nor strengthens the residual span is predicted
to have no effect on conceptual growth, even if it is extensive or vivid. This
distinguishes productive imagination from unconstrained generative activity.

Together, these signatures render the framework empirically and conceptually
testable. They locate the source of novelty not in arbitrary generation, but in
the interaction between representational failure, simplicity constraints, and
counterfactual exploration.

\subsection{Relation to Causal Counterfactuals}

It is important to distinguish the representational counterfactuals studied
here from causal counterfactuals in the sense of Pearlian interventionist
models. In causal inference, counterfactuals are defined over alternative values
of variables within a fixed representational schema, and are used to reason
about hypothetical outcomes under interventions.

By contrast, representational counterfactuals concern hypothetical extensions
or reconfigurations of the representational space itself. They ask not
``What would have happened if $X$ had been different?'' but rather ``What would
be representable if the conceptual basis were extended in this direction?''

These two notions are complementary rather than competing. Causal
counterfactuals operate within a conceptual space; representational
counterfactuals govern how that space may itself change. The present framework
thus occupies a meta-level relative to standard causal reasoning, specifying
the structural preconditions under which new variables, distinctions, or
dimensions can meaningfully enter a model.

\section{Limitations}

The framework also has clear limitations.

\begin{itemize}[leftmargin=*]
  \item The MDL objective depends on a chosen encoding of experience. Different
        encodings may induce different residual geometries, and the theory does
        not prescribe a unique or privileged representation.
  \item The analysis assumes linear, finite-dimensional conceptual spaces and
        orthogonal projections. Extensions to nonlinear, hierarchical, or
        infinite-dimensional representations remain open.
  \item The model characterizes the structural conditions for conceptual growth,
        but does not address phenomenological aspects of understanding, such as
        subjective meaning, affect, or conscious experience.
\end{itemize}

These limitations reflect a deliberate focus on representational structure
rather than psychological completeness.

\section{Conclusion}

This paper modeled conceptual growth as admissible basis extension selected by
reduction in description length, and distinguished representational
counterfactuals from causal counterfactuals. Within this framework, learning and
imagination are unified as constrained processes operating over the geometry of
residual error.

The resulting picture is conservative but precise. Novelty does not arise from
arbitrary invention, but from structurally informative failure of existing
representations. Imagination functions not as an unconstrained generator of new
ideas, but as a mechanism for exploring and amplifying residual structure.

By grounding epistemic growth in representational geometry and compression
pressure, the framework provides a formal account of how new concepts can
emerge without abandoning rational constraint, and clarifies the limits within
which creativity, learning, and counterfactual reasoning can operate.

%% file: 99-appendix.tex
\appendix

\section{Table of Symbols}
For ease of reference, we summarize the main symbols used throughout the paper.

\begin{center}
\begin{tabular}{ll}
\hline
\textbf{Symbol} & \textbf{Meaning} \\
\hline
$\mathcal H$ & Finite-dimensional Euclidean (inner-product) space of representations \\

$\mathcal C \subseteq \mathcal H$ & Current conceptual subspace of the agent \\

$\mathcal C'$ & An extended conceptual subspace (after basis extension) \\

$\mathcal C^\perp$ & Orthogonal complement of $\mathcal C$ in $\mathcal H$ \\

$D \subseteq \mathcal H$ & Finite multiset of experience vectors \\

$u \in D$ & An individual experience vector \\

$\Pi_{\mathcal S}$ & Orthogonal projection onto subspace $\mathcal S$ \\

$r_{\mathcal C}(u)$ & Residual of $u$ relative to $\mathcal C$, defined as $u-\Pi_{\mathcal C}u$ \\

$W$ & Residual span $\mathrm{span}\{r_{\mathcal C}(u) : u \in D\} \subseteq \mathcal C^\perp$ \\

$U \subseteq \mathcal C^\perp$ & Novelty subspace added to $\mathcal C$ \\

$\mathcal T(\mathcal C)$ & Class of admissible extensions of $\mathcal C$ \\

$r$ & Upper bound on admissible novelty rank \\

$\dim(\mathcal S)$ & Dimension of subspace $\mathcal S$ \\

$\operatorname{rank}(\cdot)$ & Rank of a linear operator \\

$L(\mathcal S;D)$ & MDL objective evaluated for subspace $\mathcal S$ on data $D$ \\

$\ell(\cdot)$ & Nondecreasing loss function applied to squared residual norms \\

$\lambda$ & Positive complexity penalty parameter in the MDL objective \\

$\widehat{\mathcal C}$ & Residual-supported extension $\mathcal C \oplus (U \cap W)$ \\

$W^\perp$ & Orthogonal complement of the residual span $W$ \\

\hline
\end{tabular}
\end{center}

\section{Proofs and Technical Conditions}
\label{app:ptc-proof}

\subsection{Proof of Theorem~\ref{thm:cbe}}

\begin{definition}[MDL objective and residual span]
Let $\mathcal H$ be a finite-dimensional Euclidean space and
$\mathcal C\subseteq\mathcal H$ a conceptual subspace. Let $D\subseteq\mathcal H$
be a finite multiset.

Define
\[
L(\mathcal S;D)
\;=\;
\sum_{u\in D}\ell\!\bigl(\|u-\Pi_{\mathcal S}u\|^2\bigr)
\;+\;
\lambda\,\dim(\mathcal S),
\]
where $\lambda>0$ and $\ell:\mathbb R_{\ge 0}\to\mathbb R_{\ge 0}$ is nondecreasing.

Define the residual map $r_{\mathcal C}(u):=u-\Pi_{\mathcal C}u$ and the residual span
\[
W:=\mathrm{span}\{\,r_{\mathcal C}(u):u\in D\,\}\subseteq \mathcal C^\perp.
\]
\end{definition}

\begin{definition}[Admissible extension class]
\label{def:admissible_extensions}
Let $\mathcal T(\mathcal C)$ be a class of admissible extensions of $\mathcal C$
such that every $\mathcal C'\in\mathcal T(\mathcal C)$ admits an orthogonal decomposition
\[
\mathcal C'=\mathcal C\oplus U,
\qquad U\subseteq \mathcal C^\perp .
\]
We require that:
\begin{enumerate}
\item \emph{(Bounded novelty rank)} There exists $r\in\mathbb N$ such that for all
$\mathcal C'\in\mathcal T(\mathcal C)$,
\[
\operatorname{rank}(\Pi_{\mathcal C'}-\Pi_{\mathcal C})\le r .
\]

\item \emph{($W$-reducibility and closure)} For every admissible $U$ arising from
$\mathcal C'=\mathcal C\oplus U\in\mathcal T(\mathcal C)$,
\[
U = (U\cap W)\oplus (U\cap W^\perp),
\]
and the restricted extension
\[
\widehat{\mathcal C}:=\mathcal C\oplus (U\cap W)
\]
also lies in $\mathcal T(\mathcal C)$.
\end{enumerate}
\end{definition}

\begin{proposition}[Low-rank novelty]
\label{prop:low_rank_novelty}
Under Definition~\ref{def:admissible_extensions}, for any admissible extension
$\mathcal C'=\mathcal C\oplus U\in\mathcal T(\mathcal C)$,
\[
\dim(\mathcal C')-\dim(\mathcal C)
\;=\;
\dim(U)
\;=\;
\operatorname{rank}(\Pi_{\mathcal C'}-\Pi_{\mathcal C})
\;\le\; r.
\]
\end{proposition}

\begin{proof}
Because the decomposition $\mathcal C'=\mathcal C\oplus U$ is orthogonal, the
associated projection satisfies
\[
\Pi_{\mathcal C'}=\Pi_{\mathcal C}+\Pi_U,
\]
and hence
\[
\Pi_{\mathcal C'}-\Pi_{\mathcal C}=\Pi_U.
\]
Since $\Pi_U$ is the orthogonal projection onto $U$, its rank equals $\dim(U)$.
The claimed bound then follows directly from
Definition~\ref{def:admissible_extensions}(1).
\end{proof}
From Proposition~\ref{prop:low_rank_novelty}, adding novelty directions from $\mathcal C$ to $\mathcal C'$ changes the model
only through the new projection component $\Pi_U$; the rank of this change
counts exactly how many independent degrees of freedom are added, so a bound on
the projection’s rank directly bounds the dimensionality of admissible novelty.

\begin{lemma}[Fit term depends only on $U\cap W$]
\label{lem:fit_depends_on_W}
Let $\mathcal C'=\mathcal C\oplus U$ with $U\subseteq \mathcal C^\perp$ and define
$\widehat{\mathcal C}:=\mathcal C\oplus (U\cap W)$.
Then for every $u\in D$,
\[
\Pi_{\mathcal C'}u=\Pi_{\widehat{\mathcal C}}u,
\]
and hence the data-fit terms in $L(\mathcal C';D)$ and $L(\widehat{\mathcal C};D)$ coincide.
\end{lemma}

\begin{proof}
Fix $u\in D$ and write $u=\Pi_{\mathcal C}u+r_{\mathcal C}(u)$ with $r_{\mathcal C}(u)\in W$.
Because $U\subseteq\mathcal C^\perp$,
\[
\Pi_{\mathcal C'}u=\Pi_{\mathcal C}u+\Pi_U r_{\mathcal C}(u).
\]
Using $U=(U\cap W)\oplus(U\cap W^\perp)$ and $r_{\mathcal C}(u)\in W$ gives
$\Pi_U r_{\mathcal C}(u)=\Pi_{U\cap W}r_{\mathcal C}(u)$, so
$\Pi_{\mathcal C'}u=\Pi_{\widehat{\mathcal C}}u$.
\end{proof}

From Lemma~\ref{lem:fit_depends_on_W}, the data can only “see” novelty directions that overlap with the residuals left
unexplained by the current concept $\mathcal C$; any component of $U$ orthogonal
to the residual span $W$ never affects projections of the observed data and
therefore cannot change the fit term.

\begin{proposition}[Residual-supported dominance]
\label{prop:residual_supported_dominance}
Under Definition~\ref{def:admissible_extensions}, let $\mathcal C'=\mathcal C\oplus U\in\mathcal T(\mathcal C)$
and define $\widehat{\mathcal C}:=\mathcal C\oplus(U\cap W)$.
Then $\widehat{\mathcal C}\in\mathcal T(\mathcal C)$ and
\[
L(\widehat{\mathcal C};D)\le L(\mathcal C';D).
\]
In particular, if $\mathcal C'$ is MDL-accepted (i.e.\ $L(\mathcal C';D)<L(\mathcal C;D)$),
then $\widehat{\mathcal C}$ is also MDL-accepted and has novelty contained in $W$.
\end{proposition}

\begin{proof}
By Definition~\ref{def:admissible_extensions}(2), $\widehat{\mathcal C}\in\mathcal T(\mathcal C)$.
By Lemma~\ref{lem:fit_depends_on_W}, the fit terms coincide.
Also $\dim(\widehat{\mathcal C})\le \dim(\mathcal C')$ since $U\cap W\subseteq U$,
so the complexity term weakly decreases. Hence
$L(\widehat{\mathcal C};D)\le L(\mathcal C';D)$.
\end{proof}

From Proposition~\ref{prop:residual_supported_dominance}, any admissible novelty that improves MDL can be stripped of directions that do not interact with the observed residuals without affecting data fit, leaving a
simpler extension whose novel dimensions are supported entirely on what the
data actually reveal.

\begin{proposition}[No MDL gain from purely $W^\perp$ novelty]
\label{prop:no_gain_W_perp}
Under Definition~\ref{def:admissible_extensions}, if
$\mathcal C'=\mathcal C\oplus U\in\mathcal T(\mathcal C)$ with
$U\subseteq W^\perp$, then
\[
L(\mathcal C';D)=L(\mathcal C;D)+\lambda\,\dim(U)\ge L(\mathcal C;D),
\]
with strict inequality whenever $U\neq\{0\}$.
In particular, no nontrivial extension whose novelty lies entirely in $W^\perp$
can be MDL-accepted.
\end{proposition}

\begin{proof}
If $U\subseteq W^\perp$, then for every $u\in D$ we have $r_{\mathcal C}(u)\in W$,
and hence $\Pi_U r_{\mathcal C}(u)=0$.
Since $\mathcal C'=\mathcal C\oplus U$ with $U\subseteq\mathcal C^\perp$, it follows that
\[
\Pi_{\mathcal C'}u
=\Pi_{\mathcal C}u+\Pi_U r_{\mathcal C}(u)
=\Pi_{\mathcal C}u,
\]
so the data-fit term in the MDL objective is unchanged.
The complexity penalty increases by $\lambda\,\dim(U)$, which is strictly positive
whenever $U\neq\{0\}$.
\end{proof}
From Proposition~\ref{prop:no_gain_W_perp}, novelty directions orthogonal to the residual span $W$ are invisible to the data,
so they cannot improve the fit term and only increase model complexity, making
any such extension strictly worse under the MDL objective.

\begin{theorem}[Counterfactual Basis Extension under MDL]
\label{thm:cbe_mdl}
Under Definition~\ref{def:admissible_extensions}, for any MDL-accepted update
$\mathcal C'\in\mathcal T(\mathcal C)$:
\begin{enumerate}
\item the novelty dimension satisfies $\dim(\mathcal C')-\dim(\mathcal C)\le r$
(by Proposition~\ref{prop:low_rank_novelty});
\item there exists an MDL-accepted admissible update $\widehat{\mathcal C}=\mathcal C\oplus U'$
with $U'\subseteq W$ and $L(\widehat{\mathcal C};D)\le L(\mathcal C';D)$
(by Proposition~\ref{prop:residual_supported_dominance});
\item no nontrivial extension with novelty contained in $W^\perp$ can be MDL-accepted
(by Proposition~\ref{prop:no_gain_W_perp}).
\end{enumerate}
\end{theorem}

\paragraph{Interpretation: The Geometry of Epistemic Growth.}
Theorem~\ref{thm:cbe_mdl} provides the technical foundation for the main-text
Theorem~\ref{thm:cbe}, formalizing learning as a geometric process driven by
structured representational failure rather than by passive accumulation of
information. Conceptual growth occurs only through the resolution of residual
error relative to the agent’s current conceptual subspace.

The central object of the analysis is the residual span $W$, generated by
vectors of the form $u-\Pi_{\mathcal C}u$. Geometrically, $W$ captures precisely
those directions in experience that cannot be represented within the current
conceptual space; epistemically, it characterizes the structured content of the
agent’s ignorance. The theorem shows that admissible conceptual extensions
cannot introduce genuinely novel directions arbitrarily: any MDL-accepted
extension must draw its new basis directions from $W$, while extensions
orthogonal to $W$ strictly increase description length without improving
explanatory power. In particular, when an agent’s predictions incur no residual
error, conceptual expansion is impossible.

This result sharply constrains the space of conceptual discovery. New concepts
are not introduced ex nihilo, but inherit their structure from systematic
mismatches between prediction and experience. The MDL criterion enforces
parsimony by allowing representational growth only when residual patterns are
sufficiently coherent and informative to justify the cost of increased
complexity, thereby formally distinguishing structured anomalies from
unstructured noise.

Finally, the result applies equally when experience vectors arise from internal
simulation rather than external observation. Counterfactual reasoning or
imagination contributes to learning only insofar as it generates residuals that
enlarge or refine the residual span $W$. Simulated experiences that are already
explained by the current model, or that introduce purely unstructured error, do
not support conceptual growth.

Taken together, Theorem~\ref{thm:cbe} characterizes epistemic development as a
counterfactual, error-driven process of basis extension: new concepts emerge
precisely at the boundaries of representational failure, and only to the extent
that such failures exhibit learnable geometric structure.

\subsection{Proof of Proposition~\ref{prop:mdl-threshold-general}}

\begin{proposition}[MDL Acceptance Threshold (general loss)]
Let $\mathcal{C}'=\mathcal{C}\oplus \mathrm{span}\{b_{\mathrm{new}}\}$ be an admissible
one-dimensional extension, and set $\Delta:=\lambda$. Then $\mathcal{C}'$ is
MDL-accepted (i.e., $L(\mathcal{C}';D)<L(\mathcal{C};D)$) if and only if
\[
\sum_{u\in D}\Bigl[\ell\!\bigl(\|r_{\mathcal{C}}(u)\|^2\bigr)
-
\ell\!\bigl(\|r_{\mathcal{C}'}(u)\|^2\bigr)\Bigr]
\;>\;\Delta.
\]
\end{proposition}

\begin{proof}
Because $\dim(\mathcal{C}')=\dim(\mathcal{C})+1$, the complexity term increases by
$\lambda=\Delta$. Hence
\begin{align*}
L(\mathcal{C};D)-L(\mathcal{C}';D)
&=
\sum_{u\in D}\ell\!\bigl(\|r_{\mathcal{C}}(u)\|^2\bigr)
-
\sum_{u\in D}\ell\!\bigl(\|r_{\mathcal{C}'}(u)\|^2\bigr)
-\Delta \\
&=
\sum_{u\in D}\Bigl[\ell\!\bigl(\|r_{\mathcal{C}}(u)\|^2\bigr)
-
\ell\!\bigl(\|r_{\mathcal{C}'}(u)\|^2\bigr)\Bigr]
-\Delta.
\end{align*}
Therefore $L(\mathcal{C}';D)<L(\mathcal{C};D)$ holds if and only if the bracketed sum
exceeds $\Delta$.
\end{proof}

\paragraph{Intuition.}
The proposition makes explicit that MDL acceptance is governed by a simple
trade-off between explanatory gain and representational cost. Introducing a
new basis direction incurs a fixed complexity penalty $\Delta=\lambda$,
independent of the data. An extension is therefore accepted if and only if the
aggregate reduction in residual error—measured through the loss function
$\ell$—exceeds this fixed cost. In this sense, MDL learning implements a
threshold rule: new conceptual dimensions are added only when they explain
enough previously unresolved structure in experience to justify their
descriptive overhead. This formalizes the idea that conceptual growth is
selective rather than automatic, favoring structured, systematic error
reduction over marginal or noisy improvements.

\subsection{Proof of Proposition~\ref{prop:two-mechanisms-imagination1}}

\begin{proposition}[Two Mechanisms by which Imagination Affects MDL-Based Learning]
\label{prop:two-mechanisms-imagination}
Fix a current conceptual subspace $\mathcal C\subseteq\mathcal H$.
Let $D_{\mathrm{ext}}$ denote external experience and
$D_{\mathrm{sim}}$ internally simulated experience, both finite multisets.
Write $D_{\mathrm{all}}:=D_{\mathrm{ext}}\uplus D_{\mathrm{sim}}$ and define the
residual spans (with respect to $\mathcal C$)
\[
W_{\mathrm{ext}}
:=\mathrm{span}\{r_{\mathcal C}(u):u\in D_{\mathrm{ext}}\},
\qquad
W_{\mathrm{all}}
:=\mathrm{span}\{r_{\mathcal C}(u):u\in D_{\mathrm{all}}\}.
\]
Then internal simulation can influence MDL-driven conceptual growth in exactly
two distinct ways:
\begin{enumerate}
\item \emph{(Directional enrichment)}  
If $W_{\mathrm{all}}\supsetneq W_{\mathrm{ext}}$, internal simulation enlarges the
set of residual-supported \emph{candidate novelty directions} relative to those
available from external experience alone.

\item \emph{(Threshold amplification)}  
Even if $W_{\mathrm{all}}=W_{\mathrm{ext}}$, internal simulation may increase the
aggregate reduction in description length along existing residual directions,
thereby changing the MDL acceptance decision for an admissible update.
\end{enumerate}
\end{proposition}

\begin{proof}
We begin by noting that $D_{\mathrm{ext}}\subseteq D_{\mathrm{all}}$ (as multisets)
implies $W_{\mathrm{ext}}\subseteq W_{\mathrm{all}}$.

\paragraph{(1) Directional enrichment.}
Assume $W_{\mathrm{all}}\supsetneq W_{\mathrm{ext}}$. Then there exists
$v\in W_{\mathrm{all}}\setminus W_{\mathrm{ext}}$, that is, a residual-supported
direction present in the combined dataset but absent from the residual span
induced by external experience alone.

By Theorem~\ref{thm:cbe} applied to $D_{\mathrm{all}}$, whenever an admissible
extension is MDL-accepted on $D_{\mathrm{all}}$, there exists an MDL-accepted
admissible update whose novelty subspace can be chosen to lie entirely within
$W_{\mathrm{all}}$. Since $W_{\mathrm{all}}$ strictly contains $W_{\mathrm{ext}}$,
the set of residual-supported \emph{candidate} novelty directions permitted by
the data is strictly larger than under $D_{\mathrm{ext}}$ alone. This constitutes
directional enrichment.

\paragraph{(2) Threshold amplification.}
Assume $W_{\mathrm{all}}=W_{\mathrm{ext}}$. Let
$\mathcal C'=\mathcal C\oplus\mathrm{span}\{b_{\mathrm{new}}\}$ be any admissible
one-dimensional extension, and define the aggregate improvement functional for a
dataset $D$ by
\[
G_D(\mathcal C')
:=
\sum_{u\in D}
\Bigl[
\ell(\|r_{\mathcal C}(u)\|^2)
-
\ell(\|r_{\mathcal C'}(u)\|^2)
\Bigr].
\]
By Proposition~\ref{prop:mdl-threshold-general}, $\mathcal C'$ is MDL-accepted on
$D$ if and only if $G_D(\mathcal C')>\lambda$.

Since $D_{\mathrm{all}}=D_{\mathrm{ext}}\uplus D_{\mathrm{sim}}$, additivity yields
\[
G_{D_{\mathrm{all}}}(\mathcal C')
=
G_{D_{\mathrm{ext}}}(\mathcal C')
+
G_{D_{\mathrm{sim}}}(\mathcal C').
\]
Because $\mathcal C\subseteq\mathcal C'$, we have
$\|r_{\mathcal C'}(u)\|\le \|r_{\mathcal C}(u)\|$ for all $u$, and since $\ell$ is
nondecreasing, each summand in $G_{D_{\mathrm{sim}}}(\mathcal C')$ is nonnegative.
Thus
\[
G_{D_{\mathrm{all}}}(\mathcal C')\ge G_{D_{\mathrm{ext}}}(\mathcal C'),
\]
with strict inequality whenever simulation produces at least one instance for
which the extension strictly reduces residual norm.

Consequently, it is possible that an admissible update fails to meet the MDL
acceptance threshold on external experience alone, but becomes accepted once
simulated experience is included. This effect does not require expansion of the
residual span, only an increase in the aggregate magnitude of residual-supported
evidence.
\end{proof}

\paragraph{Intuition.}
The proposition clarifies how imagination can influence learning under an
MDL criterion without appealing to any special cognitive primitives. Internal
simulation matters only insofar as it alters the geometry of residual error.
First, imagination may reveal \emph{new directions of failure} that are absent
from external experience, thereby enlarging the space of admissible conceptual
extensions. Second, even when imagination introduces no new directions, it can
\emph{reinforce existing ones} by accumulating additional evidence along
previously identified residual axes, pushing an otherwise marginal extension
past the MDL acceptance threshold. In this sense, imagination contributes either
by expanding the set of conceivable concepts or by strengthening the case for
adopting concepts that were already geometrically available but statistically
under-supported.